\documentclass{article}

\usepackage[utf8]{inputenc} 
\usepackage[T1]{fontenc}    
\usepackage{url}            
\usepackage{booktabs}       
\usepackage{amsfonts}       
\usepackage{nicefrac}       
\usepackage{microtype}      
\usepackage{dsfont}
\usepackage{amssymb}
\usepackage{amsmath}
\usepackage{amsthm}
\usepackage{color}
\usepackage{algorithmic}
\usepackage{algorithm}
\usepackage{graphicx}
\usepackage{multirow}
\usepackage{enumerate}
\usepackage{verbatim}
\usepackage{hyperref}
\usepackage{microtype}
\usepackage[letterpaper, margin=1.5in]{geometry}

\newtheorem{theorem}{Theorem}[section]
\newtheorem{definition}[theorem]{Definition}
\newtheorem{lemma}[theorem]{Lemma}

\newtheorem{corollary}[theorem]{Corollary}

\newcommand{\calX}{\mathcal{X}}

\newcommand{\adv}{x_{adv}}

\newcommand{\yz}[1]{\textcolor{black}{#1}}
\def\bbE{\mathop{\mathbb{E}}}
\def\supp{\text{supp}}

\newcommand{\raccuracy}{astuteness}
\newcommand{\racc}{Ast}

\newcommand{\raccurate}{astute}
\def\E{\mathbb{E}}
\def\calX{\mathcal{X}}
\def\E{\mathbb{E}}
\def\mone{\mathbf{1}}

\def\bbR{\mathbb{R}}
\def\calH{\mathcal{H}}

\title{Analyzing the Robustness of Nearest Neighbors to Adversarial Examples}

\author{
  Yizhen~Wang \\
  University of California, San Diego\\
  \texttt{yiw248@eng.ucsd.edu} \\
  \and
  Somesh~Jha \\
  University of Wisconsin-Madison \\
  \texttt{jha@cs.wisc.edu} \\
  \and
  Kamalika~Chaudhuri \\
  University of California, San Diego\\
  \texttt{kamalika@cs.eng.ucsd.edu} \\
}

\date{}
\begin{document}
\maketitle

\begin{abstract}

Motivated by safety-critical applications, test-time attacks on classifiers via
adversarial examples has recently received a great deal of attention. However,
there is a general lack of understanding on why adversarial examples arise;
whether they originate due to inherent properties of data or due to lack of
training samples remains ill-understood. In this work, we introduce a theoretical
framework analogous to bias-variance theory for understanding these effects.
We use our framework to analyze the robustness of a canonical non-parametric
classifier -- the $k$-nearest neighbors. Our analysis shows that its robustness
properties depend critically on the value of $k$ -- the classifier may be inherently
non-robust for small k, but its robustness approaches that of the Bayes Optimal
classifier for fast-growing $k$. We propose a novel modified $1$-nearest neighbor
classifier, and guarantee its robustness in the large sample limit. Our
experiments \footnote{Code available at: \url{https://github.com/EricYizhenWang/robust_nn_icml}} 
suggest that this classifier may have good robustness properties
even for reasonable data set sizes.
\end{abstract}

\section{Introduction}

Machine learning is increasingly applied in security-critical domains
such as automotive systems, healthcare, finance and robotics. To
ensure safe deployment in these applications, there is an increasing
need to design machine-learning algorithms that are robust in the
presence of adversarial attacks.

A realistic attack paradigm that has received a lot of recent
attention~\cite{goodfellow2014explaining,papernot2015limitations,szegedy2013intriguing,papernot2017}
is test-time attacks via {\em{adversarial examples}}. Here, an
adversary has the ability to provide modified test inputs to an
already-trained classifier, but cannot modify the training process in
any way.  Their goal is to perturb legitimate test inputs by a ``small
amount'' in order to force the classifier to report an incorrect
label. An example is an adversary that replaces a stop sign by a
slightly defaced version in order to force an autonomous vehicle to
recognize it as an yield sign.  This attack is undetectable to the
human eye if the perturbation is small enough.

Prior work has considered adversarial examples in the context of
linear classifiers~\cite{MeekLowd2005}, kernel SVMs~\cite{biggio2013evasion} and neural networks~\cite{szegedy2013intriguing, goodfellow2014explaining,
  papernot2017, papernot2015limitations, deepfool}. However, most of
this work has either been empirical, or has focussed on developing
theoretically motivated attacks and defenses. Consequently, there is a
general lack of understanding on why adversarial examples arise;
whether they originate due to inherent properties of data or due to
lack of training samples remains ill-understood.

This work develops a theoretical framework for robust learning in order to
understand the effects of distributional properties and finite samples on
robustness. Building on traditional bias-variance
theory~\cite{FriedmanHastieTibshirani2000}, we posit that a classification
algorithm may be robust to adversarial examples due to three reasons.  First,
it may be {\it distributionally robust}, in the sense that the output
classifier is robust as the number of training samples grow to infinity.
Second, even the output of a distributionally robust classification algorithm
may be vulnerable due to too few training samples -- this is characterized by
finite sample robustness. Finally, different training algorithms might result
in classifiers with different degrees of robustness, which we call {\it
algorithmic robustness}. These quantities are analogous to bias, variance and
algorithmic effects respectively.

Next, we analyze a simple non-parametric classification algorithm:
{\it $k$-nearest neighbors} in our framework. Our analysis
demonstrates that large sample robustness properties of this algorithm
depend very much on $k$.

Specifically, we identify two distinct regimes for $k$ with vastly different
robustness properties. When $k$ is constant, we show that $k$-nearest neighbors
has zero robustness in the large sample limit in regions where $p(y=1|x)$ lies
in $(0, 1)$. This is in contrast with accuracy, which may be quite high in
these regions. For $k = \Omega(\sqrt{d n \log n})$, where $d$ is the data
dimension and $n$ is the sample size, we show that the robustness region of
$k$-nearest neighbors approaches that of the Bayes Optimal classifier in the
large sample limit. This is again in contrast with accuracy, where convergence
to the Bayes Optimal accuracy is known for a much slower growing
$k$~\cite{Lugosi, CD14}.

Since $k = \Omega(\sqrt{d n \log n})$ is too high to use in
practice with nearest neighbors, we next propose a novel robust version of the $1$-nearest neighbor
classifier that operates on a modified training set.  We provably show that
in the large sample limit, this algorithm has superior robustness to standard
$1$-nearest neighbors for data distributions with certain
properties. 

Finally, we validate our theoretical results by empirically evaluating our
algorithm on three datasets against several popular attacks. Our experiments
demonstrate that our algorithm performs better than or about as well as both
standard $1$-nearest neighbors and nearest neighbors with adversarial training -- a
popular and effective defense mechanism. This suggests that although our
performance guarantees hold in the large sample limit, our algorithm may have
good robustness properties even for realistic training data sizes. 

\subsection{Related Work}

Adversarial examples have recently received a great deal of
attention~\cite{goodfellow2014explaining,biggio2013evasion,papernot2015limitations,szegedy2013intriguing,papernot2017}.
Most of the work, however, has been empirical, and has focussed on developing
increasingly sophisticated attacks and defenses.

\subsubsection{Related Work on Adversarial Examples}

Prior theoretical work on adversarial examples falls into two
categories -- analysis and theory-inspired defenses. Work on analysis
includes~\cite{nips2016thy}, which analyzes the robustness of linear
and quadratic classifiers under random and semi-random perturbations.
\cite{Hein2017} provides robustness guarantees on linear and kernel
classifiers trained on a given data set. \cite{advspheres} shows that
linear classifiers for high dimensional datasets may have inherent
robustness-accuracy trade-offs.

Work on theory-inspired defenses include~\cite{Madry2017, kolter2017,
duchi2017}, who provide defense mechanisms for adversarial examples in neural
networks that are relaxations of certain principled optimization objectives.
\cite{barett} shows how to use program verification to certify robustness of
neural networks around given inputs for small neural networks. 

Our work differs from these in two important ways. First, unlike most prior work
which looks at a given training dataset, we consider effects of the data
distribution and number of samples, and analyze robustness properties in the large
sample limit. Second, unlike prior work which largely focuses on parametric
methods such as neural networks, our focus is on a canonical non-parametric
method -- the nearest neighbors classifier. 

\subsubsection{Related Work on Nearest Neighbors}

There has been a body of work on the convergence and consistency of
nearest-neighbor classifiers and their many variants~\cite{CH67, S77, KP95,
devroyebook, CD14, kontorovich1}; all these works however
consider accuracy and not robustness.

In the asymptotic regime, \cite{CH67} shows that the accuracy of $1$-nearest
neighbors converges in the large sample limit to $1-2R^*(1 - R^*)$ where $R^*$ is
the error rate of the Bayes Optimal classifier. This implies that even
$1$-nearest neighbor may achieve relatively high accuracy even when $p(y=1|x)$
is not $0$ or $1$. In contrast, we show that $1$-nearest neighbor is
{\em{inherently non-robust}} when $p(y=1|x) \in (0, 1)$ under some continuity conditions.

For larger $k$, the accuracy of $k$-nearest neighbors is known to converge to
that of the Bayes Optimal classifier if $k_n \rightarrow \infty$ and $k_n/n
\rightarrow 0$ as the sample size $n \rightarrow \infty$. We show that the
robustness also converges to that of the Bayes Optimal classifier when $k_n$
grows at a much higher rate -- fast enough to ensure uniform convergence.
Whether this high rate is necessary remains an intriguing open question.

Finite sample rates on the accuracy of nearest neighbors are known
to depend heavily on properties of the data distribution, and there is no
distribution free rate as in parametric methods~\cite{devroyebook}. \cite{CD14}
provides a clean characterization of the finite sample rates of nearest
neighbors as a function of natural {\em{interiors}} of the classes. Here we
build on their results by defining a stricter, more robust version of interiors
and providing bounds as functions of these new robust quantities.

\subsubsection{Other Related Work}
~\cite{NTT} provides a method for generating adversarial examples for nearest
neighbors, and shows that the effectiveness of attacks grow with intrinsic
dimensionality.  Finally,~\cite{papernot2016transferability, papernot2017}
provides {\em{black-box}} attacks on substitute classifiers; their
experiments show that attacks from other types of substitute classifiers are
not successful on nearest neighbors; our experiments corroborate these results.

\section{The Setting and Definitions}
\label{sec:setting}

\subsection{The Basic Setup}

We consider test-time attacks in a white box setting, where the
adversary has full knowledge of the training process -- namely, the
type of classifier used, the training data and any parameters -- but
cannot modify training in any way.

Given an input $x$, the adversary's goal is to perturb it so as to
force the trained classifier $f$ to report a different label than
$f(x)$. The amount of perturbation is measured by an
application-specific metric $d$, and is constrained to be within a
radius $r$. 
\yz{Our analysis can be extended to any metric, but for this paper we assume that $d$ is the Euclidean distance for mathematical simplicity; we also focus on binary classification, and leave extensions to multiclass for future work.}

Finally, we assume that unlabeled instances are drawn from an instance
space $\calX$, and their labels are drawn from the label space $
\{ 0, 1 \}$. There is an underlying data distribution $D$ that
generates labeled examples; the marginal over $\calX$ of $D$ is $\mu$
and the conditional distribution of labels given $x$ is denoted by
$\eta$.

\subsection{Robustness and \raccuracy}

We begin by defining robustness, which for a classifier $f$ at input $x$ is measured by the robustness radius.

\begin{definition} [Robustness Radius]
The robustness radius of a classifier $f$ at an instance $x \in
\calX$, denoted by $\rho(f, x)$, is the shortest distance between $x$
and an input $x'$ to which $f$ assigns a label different from $f(x)$:
\[ \rho(f, x) = \inf_{r} \{ \exists x' \in \calX \cap B(x, r) {\text{\ s.t\ }} f(x) \neq f(x') \} \]
\end{definition}

Observe that the robustness radius measures a classifier's local robustness. 
\yz{A classifier $f$ with robustness radius $r$ at $x$ guarantees that no adversarial example of $x$ with norm of perturbation less than $r$ can be created using any attack method.}
A plausible way to extend this into a global notion is to require a lower bound
on the robustness radius everywhere; however, only the constant classifier will
satisfy this condition. Instead, we consider robustness around
{\em{meaningful instances}}, that we model as examples drawn from the
underlying data distribution. 

\begin{definition}[Robustness with respect to a Distribution]
The robustness of a classifier $f$ at radius $r$ with respect to a
distribution $\mu$ over the instance space $\calX$, denoted by $R(f,
r, \mu)$, is the fraction of instances drawn from $\mu$ for which the
robustness radius is greater than or equal to $r$.
\[ R(f, r, \mu) = \Pr_{x \sim \mu} (\rho(f, x) \geq r) \]
\end{definition}

Finally, observe that we are interested in classifiers that are
{\em{both}} robust and accurate. This leads to the notion of
\raccuracy, which measures the fraction of instances on which a
classifier is both accurate and robust.

\begin{definition}[\raccuracy]
The \raccuracy\ of a classifier $f$ with respect to a data
distribution $D$ and a radius $r$ is the fraction of examples on which
it is accurate and has robustness radius at least $r$; formally,
\[ \racc_{D}(f, r) = \Pr_{(x, y) \sim D}( \rho(f, x) \geq r, f(x) = y ), \]
\end{definition}

Observe that \raccuracy\ is analogous to classification accuracy, and
we argue that it is a more appropriate metric if we are concerned with
both robustness and accuracy. Unlike accuracy, \raccuracy\ cannot be
directly empirically measured unless we have a way to certify a lower
bound on the robustness radius. In this work, we will prove
bounds on the \raccuracy\ of classifiers, and in our
experiments, we will approximate it by measuring resistance to
standard attacks.

\subsection{Sources of Robustness}

There are three plausible reasons why classifiers lack robustness --
{\it distributional, finite sample} and {\it algorithmic}. These
sources are analogous to bias, variance, and algorithmic effects
respectively in standard bias-variance theory.

Distributional robustness measures the effect of the data distribution
on robustness when an infinitely large number of samples are used
to train the classifier. Formally, if $S_n$ is a training sample of
size $n$ drawn from $D$ and $A(S_n, \cdot)$ is a classifier obtained by applying
the training procedure $A$ on $S_n$, then the distributional
robustness at radius $r$ is $\lim_{n \rightarrow \infty}
\bbE_{S_n \sim D} [ R(A(S_n, \cdot), r, \mu)]$.

In contrast, for finite sample robustness, we characterize the behaviour of
$R(A(S_n, \cdot), r, \mu)$ for finite $n$ -- usually by putting high
probability bounds over the training set. Thus, finite sample robustness
depends on the training set size $n$, and quantifies how it changes
with sample size.  Finally, robustness also depends on the training
algorithm itself; for example, some variants of nearest neighbors may have higher
robustness than nearest neighbors itself.

\subsection{Nearest Neighbor and Bayes Optimal Classifiers}

Given a training set $S_n = \{ (X_1, Y_1), \ldots, (X_n, Y_n)\}$ and a
test example $x$, we use the notation $X^{(i)}(x)$ to denote the
$i$-th nearest neighbor of $x$ in $S_n$, and $Y^{(i)}(x)$ to denote
the label of $X^{(i)}(x)$.

Given a test example $x$, the $k$-nearest neighbor classifier
$A_k(S_n, x)$ outputs:
\begin{eqnarray*}
& = 1, & {\text{if\ }} Y^{(1)}(x) + \ldots + Y^{(k)}(x) \geq k/2  \\
& = 0, & {\text{otherwise.}}
\end{eqnarray*}

The Bayes optimal classifier $g$ over a data distribution $D$ has the
following classification rule:
\begin{equation}
g(x) = \left\{ \begin{array}{ll}
         1 & \mbox{if $\eta(x)=\Pr(y=1|x) \geq 1/2$};\\
         0 & \mbox{otherwise}.\end{array} \right.
\end{equation}

\section{Robustness of Nearest Neighbors}
\label{sec:nnanalysis}

How robust is the $k$-nearest neighbor classifier? We show that it depends on
the value of $k$. Specifically, we identify two distinct regimes -- constant
$k$ and $k = \Omega(\sqrt{d n \log n})$ where $d$ is the data dimension -- and
show that nearest neighbors has different robustness properties in the two. 

\subsection{Low $k$ Regime}

In this region, $k$ is a constant that does not depend on the training
set size $n$. Provided certain regularity conditions hold, we show
that $k$-nearest neighbors is inherently non-robust in this regime
unless $\eta(x) \in \{ 0, 1 \}$ -- in the sense that the
distributional robustness becomes $0$ in the large sample limit.

\begin{theorem}
Let $x \in \calX \cap \supp(\mu)$ such that (a) $\mu$ is absolutely
continuous with respect to the Lebesgue measure (b) $\eta(x) \in (0,
1)$ (c) $\eta$ is continuous with respect to the Euclidean metric in a
neighborhood of $x$. Then, for fixed $k$, $\rho(A_k(S_n, \cdot), x)$
converges in probability to $0$.
\label{thm:knnlowerbound}
\end{theorem}

\paragraph{Remarks.} 
Observe that Theorem~\ref{thm:knnlowerbound} implies that the
distributional robustness (and hence astuteness) in a region where
$\eta(x) \in (0, 1)$ is $0$.  This is in contrast with accuracy;
for $1$-NN, the accuracy converges to \yz{$1 - 2 R^* (1 - R^*)$} as $n \rightarrow
\infty$, where $R^*$ is the \yz{error rate} of the Bayes Optimal classifier,
and thus may be quite high.

The proof of Theorem~\ref{thm:knnlowerbound} in the Appendix shows
that the absolute continuity of $\mu$ with respect to the Lebesgue
measure is not strictly necessary; absolute continuity with respect to
an embedded manifold will give the same result, but will result in a
more complex proof.

In the Appendix A (Theorem A.2), we show that $k$-nearest neighbor
is \raccurate\ in the interior of the region where $\eta \in \{ 0, 1 \}$, and
provide finite sample rates for this case. 

\subsection{High $k$ Regime}

Prior work has shown that in the large sample limit, the accuracy of the
nearest neighbor classifiers converge to the Bayes Optimal, provided $k$ is set
properly. We next show that if $k$ is $\Omega(\sqrt{d n \log n})$, the regions
of robustness and the \raccuracy\ of the $k$ nearest neighbor classifiers also
approach the corresponding quantities for the Bayes Optimal classifier as $n
\rightarrow \infty$. Thus, if the Bayes Optimal classifier is robust, then so
is $k$-nearest neighbors in the large sample limit. 

The main intuition is that $k = \Omega(\sqrt{d n \log n})$ is large enough for
{\em{uniform convergence}} -- where, with high probability,
{\em{all}} Euclidean balls with $k$ examples have the property that the
empirical averages of their labels are close to their expectations. This
guarantees that for any $x$, the $k$-nearest neighbor reports the same label as
the Bayes Optimal classifier for {\em{all}} $x'$ close to $x$. Thus, if the
Bayes Optimal classifier is robust, so is nearest neighbors.

\subsubsection{Definitions}

 We begin with some definitions that we can use to characterize the
 robustness of the Bayes Optimal classifier. Following~\cite{CD14},
 we use the notation $B^o(x, r)$ to denote an open ball and $B(x, r)$ to denote a closed ball of
 radius $r$ around $x$. We define the probability radius of a ball
 around $x$ as:
\[ r_p(x) = \inf\{ r \; | \; \mu(B(x, r)) \geq p \} \]
We next define the $r$-robust $(p, \Delta)$-strict interiors as follows:
\begin{eqnarray} \label{eqn:defint}
    \calX^+_{r, \Delta, p} & = & \{ x \in supp(\mu) \; | \; \forall x' \in B^o(x, r), \nonumber \\ && \forall x'' \in B(x', r_p(x')),  \eta(x'') > 1/2 + \Delta \} \nonumber \\
    \calX^-_{r, \Delta, p} & = & \{ x \in supp(\mu) \; | \; \forall x' \in B^o(x, r), \nonumber \\ && \forall x'' \in B(x', r_p(x')),  \eta(x'') < 1/2 - \Delta \} \nonumber
\end{eqnarray}

What is the significance of these interiors? Let $x'$ be an instance such that
all $x'' \in B(x', r_p(x'))$ have $\eta(x'') > 1/2 + \Delta$. If $p \approx
\frac{k}{n}$, then the $k$ points $x''$ closest to $x'$ have $\eta(x'') > 1/2 +
\Delta$. Provided the average of the labels of these points is close to
expectation, which happens when $k$ is large relative to $1/\Delta$,
$k$-nearest neighbor outputs label $1$ on $x'$. When $x$ is in the $r$-robust
$(p, \Delta)$-strict interior region $\calX^+_{r, \Delta, p}$, this is true for
all $x'$ within distance $r$ of $x$, which means that $k$-nearest neighbors
will be robust at $x$.  Thus, the $r$-robust $(p, \Delta)$-strict
interior is the region where we natually expect $k$-nearest neighbor to have robustness
radius $r$, when $k$ is large relative to $\frac{1}{\Delta}$ and $p \approx
\frac{k}{n}$. 

Readers familiar with~\cite{CD14} will observe that the set of all
$x'$ for which $\forall x'' \in B(x', r_p(x')), \eta(x'') > 1/2 +
\Delta$ forms a stricter version of the $(p, \Delta)$-interiors of the
$1$ region that was defined in this work; these $x'$ also represent
the region where $k$-nearest neighbors are accurate when $k \approx
\max(np, 1/\Delta^2)$. The $r$-robust $(p, \Delta)$-strict interior is
thus a somewhat stricter and more robust version of this definition.

\subsubsection{Main Results}

We begin by characterizing where the Bayes Optimal classifier is
robust.

\begin{theorem}
The Bayes Optimal classifier has robustness radius $r$ at $x \in
\calX^+_{r, 0, 0} \cup \calX^-_{r, 0, 0}$. Moreover, its
\raccuracy\ is $\E[ \eta(x) \mathbf{1}(x \in \calX^+_{r, 0, 0})] +
\E[(1 - \eta(x)) \mathbf{1}(x \in \calX^-_{r, 0, 0})]$.
\label{thm:bayesoptrobustness}
\end{theorem}

The proof is in the Appendix, along with analogous results for
\raccuracy. The following theorem, along with a similar result for
\raccuracy, proved in the Appendix, characterizes robustness in the
large $k$ regime.

\begin{theorem}
For any $n$, pick a $\delta$ and a $\Delta_n \rightarrow 0$. There
exist constant $C_1$ and $C_2$ such that if $k_n \geq C_1
\frac{\sqrt{d n \log n + n \log(1/\delta_n)}}{\Delta_n}$, and $p_n
\geq \frac{k_n}{n} (1 + C_2 \sqrt{\frac{d \log n +
    \log(1/\delta)}{k_n}})$, then, with probability $\geq 1 - 3
\delta$, $k_n$-NN has robustness radius $r$ in $x \in \calX^+_{r,
  \Delta_n, p_n} \cup \calX^-_{r, \Delta_n, p_n}$.
\label{thm:knnrobustness}
\end{theorem}

\paragraph{Remarks.} 
Some remarks are in order. First, observe that as $n \rightarrow
\infty$, $\Delta_n$ and $p_n$ tend to $0$; thus, provided certain
continuity conditions hold, $\calX^+_{r, \Delta_n, p_n} \cup
\calX^-_{r, \Delta_n, p_n}$ approaches $\calX^+_{r, 0, 0} \cup
\calX^-_{r, 0, 0}$, the robustness region of the Bayes Optimal
classifier.

Second, observe that as $r$-robust strict interiors extend the
definition of interiors in~\cite{CD14},
Theorem~\ref{thm:knnrobustness} is a robustness analogue of Theorem 5
in this work. Unlike the latter, 
Theorem~\ref{thm:knnrobustness} has a more stringent requirement on
$k$. Whether this is necessary is left as an open question for future
work.

\section{A Robust 1-NN Algorithm}
\label{sec:defense}

Section~\ref{sec:nnanalysis} shows that nearest neighbors is robust for $k$ as large as $\Omega(\sqrt{d n \log n})$. However, this $k$ is too high to use in practice -- high values of $k$ require even higher sample sizes~\cite{CD14}, and lead to higher running times. Thus a natural question is whether we can find a more robust version of the algorithm for smaller $k$. In this section, we provide a more robust version of $1$-nearest neighbors, and analytically demonstrate its robustness. 

Our algorithm is motivated by the observation that $1$-nearest neighbor is robust when oppositely labeled points are far apart, and when test points lie close to training data. Most training datasets however contain nearby points that are oppositely labeled; thus, we propose to remove a subset of training points to enforce this property. 

Which points should we remove? A plausible approach is to keep the largest subset where oppositely labeled points are far apart; however, this subset has poor stability properties even for large $n$. Therefore, we propose to keep all points $x$ such that: (a) we are highly confident about the label of $x$ and its nearby points and (b) all points close to $x$ have the same label. Given that all such $x$ are kept, we remove as few points as possible, and execute nearest neighbors on the remaining dataset. 

The following definition characterizes data where oppositely labeled points are far apart. 

\begin{definition}[$r$-separated set]
A set $A = \{ (x_1, y_1), \ldots, (x_m, y_m) \}$ of labeled examples
is said to be $r$-separated if for all pairs $(x_i, y_i), (x_j, y_j)
\in A$, $\| x_i - x_j \| \leq r$ implies $y_i = y_j$.
\end{definition}

The full algorithm is described in Algorithm~\ref{alg:robust_1nn} and Algorithm~\ref{alg:confident_label}. Given confidence parameters $\Delta$ and $\delta$, Algorithm~\ref{alg:confident_label} returns a $0/1$ label when this label agrees with the average of $k_n$ points closest to $x$; otherwise, it returns $\bot$. $k_n$ is chosen such that with probability $\geq 1 - \delta$, the empirical majority of $k_n$ labels agrees with the majority in expectation, provided the latter is at least $\Delta$ away from $\frac{1}{2}$. 

Algorithm~\ref{alg:confident_label} is used to determine whether an $x_i$ should be kept. Let $f(x_i)$ be the output of Algorithm~\ref{alg:confident_label} on $x_i$.  If $y_i = f(x_i)$ and if for all $x_j \in B(x_i, r)$, $f(x_i) = f(x_j) = y_i$, then we mark $x_i$ as red. Finally, we compute the largest $r$-separated subset of the training data that includes all the red points; this reduces to a constrained matching problem as in~\cite{kontorovich1}. The resulting set, returned by Algorithm~\ref{alg:robust_1nn}, is our new training set. We observe that this set is $r$-separated from Lemma~\ref{lem:redseparated} in the Appendix, and thus oppositely labeled points are far apart. Moreover, we keep all $(x_i, y_i)$ when we are confident about the label of $x_i$ and its nearby points. \yz{Observe that our final procedure is a 1-NN algorithm, even though $k_n$ neighbors are used to determine if a point should be retained in the training set.}

\begin{algorithm}
\caption{Robust\_1NN($S_n$, $r$, $\Delta$, $\delta$, $x$)}
\begin{algorithmic}
\FOR {$(x_i, y_i)\in S_n$}
\STATE {$f(x_i) = \mbox{Confident-Label}(S_n, \Delta, \delta, x_i)$}
\ENDFOR
\STATE{$S_{RED}=\emptyset$}
\FOR {$(x_i, y_i) \in S_n$}
\IF {$f(x_i) = y_i$ and $f(x_i) = f(x_j)$ for all $x_j$ such that $\| x_i - x_j \| \leq r$ and $(x_j, y_j) \in S_n$}
\STATE {$S_{RED} = S_{RED}\bigcup \{(x_i, y_i)\}$}
\ENDIF
\ENDFOR
\STATE {Let $S'$ be the largest $2r$-separated subset of $S_n$ that contains all points in $S_{RED}$.}
\RETURN new training set $S'$
\end{algorithmic}
\label{alg:robust_1nn}
\end{algorithm}

\begin{algorithm}
\caption{Confident-Label($S_n$, $\Delta$, $\delta$, $x$)}
\begin{algorithmic}
\STATE{$k_n = 3\log(2n/\delta)/\Delta^2$}
\STATE{$\bar{y} = (1/k_n)\sum_{i=1}^{k_n}Y^{(i)}(x)$}
\IF {$\bar{y} \in [ \frac{1}{2} -\Delta, \frac{1}{2} + \Delta]$}
\RETURN{$\bot$}
\ELSE
\RETURN{$\frac{1}{2} sgn(\bar{y} - \frac{1}{2}) + \frac{1}{2}$}
\ENDIF
\end{algorithmic}
\label{alg:confident_label}
\end{algorithm}

\subsection{Performance Guarantees} 

The following theorem establishes performance guarantees for Algorithm~\ref{alg:robust_1nn}. 

\begin{theorem}\label{thm:defense}
Pick a $\Delta_n$ and $\delta$, and set $k_n = 3 \log (2n/\delta)/\Delta_n^2$. Pick a {\em{margin parameter}} $\tau$. Then, there exist constants $C$ and $C_0$ such that the following hold. If we set $p_n = \frac{k_n}{n}(1 + C \sqrt{ \frac{d \log n + \log(1/\delta)}{k_n}})$, and define the set:
\begin{eqnarray*} X_R & = & \Bigg{\{} x \Big{|} x \in \calX^+_{r + \tau, \Delta_n, p_n} \cup \calX^-_{r + \tau, \Delta_n, p_n}, \\
&& \mu(B(x, \tau)) \geq \frac{2C_0}{n}(d \log n + \log(1/\delta)) \Bigg{\}} 
\end{eqnarray*}
Then, with probability $\geq 1 - 2\delta$ over the training set, Algorithm~\ref{alg:robust_1nn} run with parameters $r$, $\Delta_n$ and $\delta$ has robustness radius at least $r - 2 \tau$ on $X_R$.
\end{theorem}

\paragraph{Remarks.} The proof is in the Appendix, along with an analogous result for \raccuracy. 
Observe that $X_R$ is roughly the {\em{high density subset}} of the $r + \tau$-robust strict interior $\calX^+_{r + \tau, \Delta_n, p_n} \cup \calX^-_{r + \tau, \Delta_n, p_n}$. Since $\eta(x)$ is constrained to be greater than $\frac{1}{2} + \Delta_n$ or less than $\frac{1}{2} - \Delta_n$ in this region, as opposed to $0$ or $1$, this is an improvement over standard nearest neighbors when the data distribution has a large high density region that intersects with the interiors. 

A second observation is that as $\tau$ is an arbitrary constant, we can set to it be quite small and still satisfy the condition on $\mu(B(x, \tau))$ for a large fraction of $x$'s when $n$ is very large. This means that in the large sample limit, $r - 2 \tau$ may be close to $r$ and $X_R$ may be close to the high density subset of $\calX^+_{r, \Delta_n, p_n} \cup \calX^-_{r, \Delta_n, p_n}$ for a lot of smooth distributions.

\section{Experiments}
\label{sec:exp}

The results in Section~\ref{sec:defense} assume large sample limits. Thus, a natural question is how well Algorithm~\ref{alg:robust_1nn} performs with more reasonable amounts of training data. We now empirically investigate this question.  

Since there are no general methods that certify robustness at an input, we assess robustness by measuring how our algorithm performs against a suite of standard attack methods. Specifically, we consider the following questions:

\begin{enumerate}
\item 
1. How does our algorithm perform against popular white box and black box attacks compared with standard baselines?
\item 
2. How is performance affected when we change the training set size relative to the data dimension?
\end{enumerate}

These questions are considered in the context of three datasets with varying training set sizes relative to the dimension, as well as two standard white box attacks and black box attacks with two kinds of substitute classifiers.

\begin{figure*}[!t]
\centering
\includegraphics[scale=0.3]{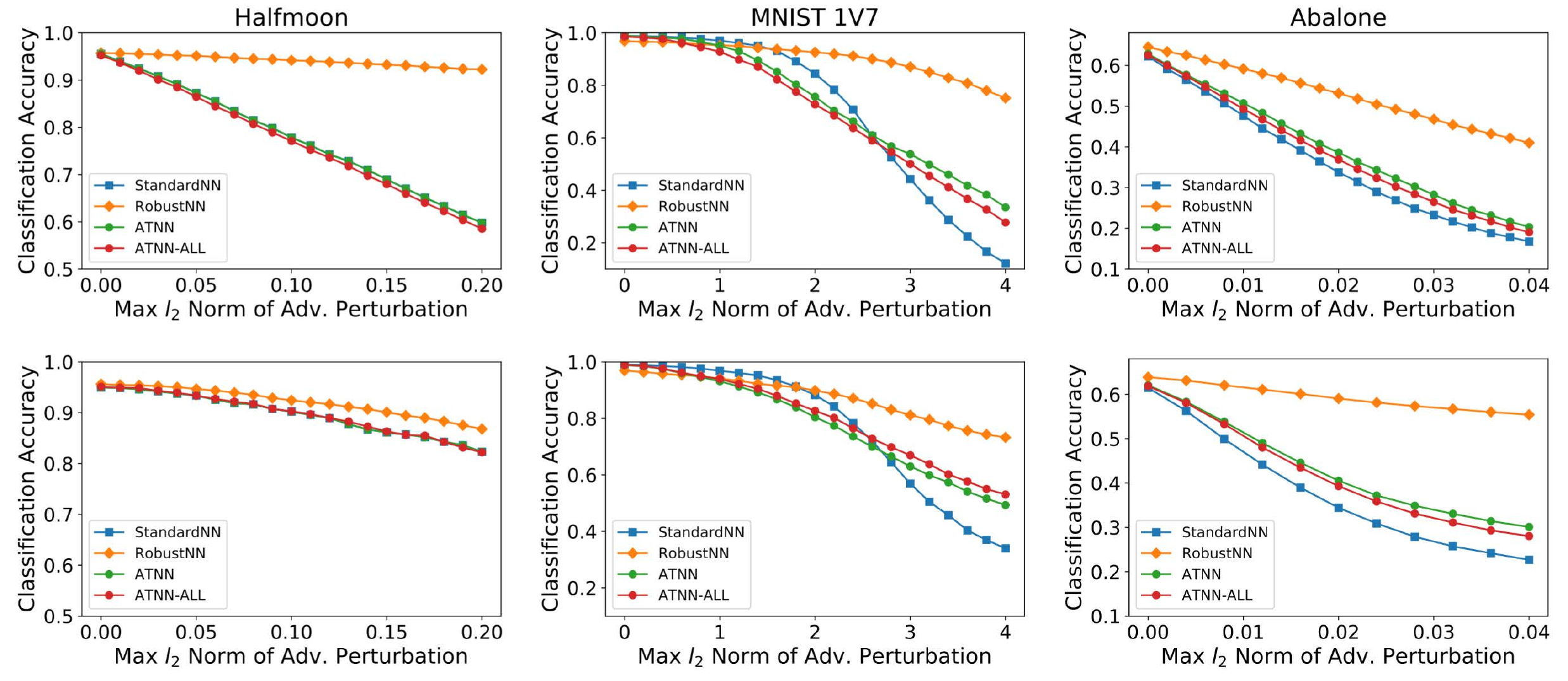}
\caption{{\textbf{White Box Attacks:}} Plot of classification accuracy on adversarial examples v.s. attack radius. \textit{Top row:} Direct Attack. \textit{Bottom row:} Kernel Substitute Attack. \textit{Left to right:} 1) Halfmoon, 2) MNIST 1v 7 and 3) Abalone.}
\label{fig:exp_wb}
\end{figure*}

\begin{figure*}[!t]
\centering
\includegraphics[scale=0.3]{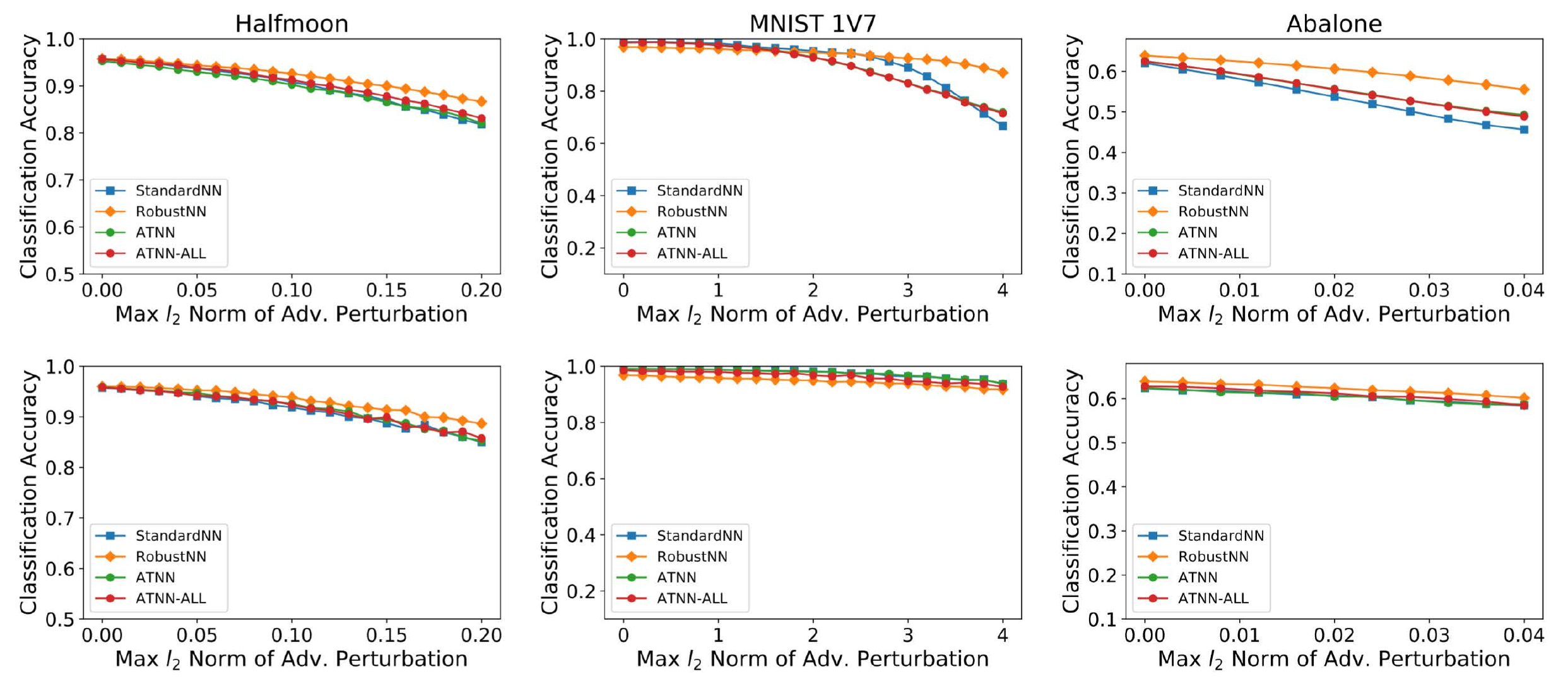}
\caption{{\textbf{Black Box Attacks:}} Plot of classification accuracy on adversarial examples v.s. attack radius. \textit{Top to Bottom:} 1) kernel substitute, 2) neural net substitute. \textit{Left to right:} 1) Halfmoon, 2) MNIST 1 v.s. 7 and 3) Abalone.}
\label{fig:exp_bb}
\end{figure*}

\subsection{Methodology}

\subsubsection{Data}
We use three datasets -- Halfmoon, MNIST 1v7 and Abalone -- with differing data sizes relative to dimension. Halfmoon is a popular $2$-dimensional synthetic data set for non-linear classification. We use a training set of size $2000$ and a test set of size $1000$ generated with standard deviation $\sigma=0.2$. The MNIST 1v7 data set is a subset of the $784$-dimensional MNIST data.  For training, we use $1000$ images each of Digit 1 and 7, and for test, $500$ images of each digit. Finally, for the Abalone dataset~\cite{UCI}, our classification task is to distinguish whether an abalone is older than 12.5 years based on $7$ physical measurements. For training, we use $500$ and for test, $100$ samples. In addition, a validation set with the same size as the test set is generated for each experiment for parameter tuning.

\subsubsection{Baselines}
We compare Algorithm~\ref{alg:robust_1nn}, denoted by RobustNN, against three baselines.  The first is the standard $1$-nearest neighbor algorithm, denoted by StandardNN. We use two forms of adversarially-trained nearest neighbors - ATNN and ATNN-all. Let $S$ be the training set used by standard nearest neighbors. In ATNN, we augment $S$ by creating, for each $(x, y) \in S$, an adversarial example $\adv$ using the attack method in the experiment, and adding $(\adv, y)$. The ATNN classifier is $1$-nearest neighbor on this augmented data. In ATNN-all, for each $(x,y) \in S$, we create adversarial examples using {\em{all}} the attack methods in the experiment, and add them all to $S$. ATNN-all is the nearest neighbor classifier on this augmented data. For example, for white box Direct Attacks in Section~\ref{sec:whitebox}, ATNN includes adversarial examples generated by the Direct Attack, and ATNN-all includes adversarial examples generated by both Direct and Kernel Substitute Attacks.

Observe that all algorithms except StandardNN have parameters to tune. RobustNN has three input parameters -- $\Delta, \delta$ and a defense radius $r$ which is an approximation to the robustness radius. For simplicity, we set $\Delta=0.45$, $\delta=0.1$ and tune $r$ on the validation set; this can be viewed as tuning the parameter $\tau$ in Theorem~\ref{thm:defense}. For ATNN and ATNN-all, the methods that generate the augmenting adversarial examples need a perturbation magnitude $r$; we call this the {\em{defense radius}}. To be fair to all algorithms, we tune the defense radius for each. We consider the adversary with the highest attack perturbation magnitude in the experiment, and select the defense radius that yields the highest validation accuracy against this adversary.

\subsection{White-box Attacks and Results}
\label{sec:whitebox}
To evaluate the robustness of Algorithm~\ref{alg:robust_1nn}, we use two standard classes of attacks -- white box and black box. For white-box attacks, the adversary knows all details about the classifier under attack, including its training data, the training algorithm and any hyperparameters.

\subsubsection{Attack Methods}

We consider two white-box attacks -- direct attack~\cite{NTT} and Kernel Substitute Attack~\cite{papernot2016transferability}.

\noindent\textbf{Direct Attack.} This attack takes as input a test example $x$, an attack radius $r$, and a training dataset $S$ (which may be an augmented or reduced dataset).  It finds an $x' \in S$ that is closest to $x$ but has a different label, and returns the adversarial example $\adv = x + r\frac{x-x'}{||x-x'||_2}$. 

\noindent\textbf{Kernel Substitute Attack.} This method attacks a substitute kernel classifier trained on the {\em{same training set}}. For a test input $\vec{x}$, a set of training points $Z$ with one-hot labels $Y$, a kernel classifier $f$ predicts the class probability as:
$$
f : \vec{x} \to \frac{\left[e^{-||\vec{x} - \vec{z}||_2^2/c}\right]_{\vec{z} \in X}}{\sum_{\vec{z} \in X}e^{-||\vec{x} - \vec{z}||_2^2/c}}\cdot Y
$$
The adversary trains a kernel classifier on the training set of the corresponding nearest neighbors, and then generates adversarial examples against this kernel classifier. The advantage is that the prediction of the kernel classifier is differentiable, which allows the use of standard gradient-based attack methods. For our experiments, we use the popular fast-gradient-sign method (FSGM). The parameter $c$ is tuned to yield the most effective attack, and is set to $0.1$ for Halfmoon and MNIST, and $0.01$ for Abalone.

\subsubsection{Results}

Figure~\ref{fig:exp_wb} shows the results. We see that RobustNN outperforms all baselines for Halfmoon and Abalone for all attack radii. For MNIST, for low attack radii, RobustNN's classification accuracy is slightly lower than the others, while it outperforms the others for large attack radii. Additionally, as is to be expected, the Direct Attack results in lower general accuracy than the Kernel Substitute Attack. 

These results suggest that our algorithm mostly outperforms the baselines StandardNN, ATNN and ATNN-all. As predicted by theory, the performance gain is higher when the training set size is large relative to the dimension -- which is the setting where nearest neighbors work well in general. It has superior performance for Halfmoon and Abalone, where the training set size is large to medium relative to dimension. In contrast, in the sparse dataset MNIST, our algorithm has slightly lower classification accuracy for small attack radii, and higher otherwise.

\subsection{Black-box Attacks and Results}

\cite{papernot2017} has observed that some defense methods that work by masking gradients remain highly amenable to {\em{black box attacks}}. In this attack, the adversary is unaware of the target classifier's nature, parameters or training data, but has access to a seed dataset drawn from the same distribution which they use to train and attack a substitute classifier. To establish robustness properties of Algorithm~\ref{alg:robust_1nn}, we therefore validate it against black box attacks based on two types of substitute classifiers. 

\subsubsection{Attack Methods} 

We use two types of substitute classifiers -- kernel classifiers and neural
networks. The adversary trains the substitute classifier using the method
of~\cite{papernot2017} and uses the adversarial examples against the substitute
to attack the target classifier.

\smallskip\noindent\textbf{Kernel Classifier.} The kernel classifier substitute is the same as the one in Section~\ref{sec:whitebox}, but trained using the seed data and the method of~\cite{papernot2017}. 

\noindent\textbf{Neural Networks.} The neural network for MNIST is the ConvNet in  \cite{papernot2017cleverhans}'s tutorial. For Halfmoon and Abalone, the network is a multi-layer perceptron with $2$ hidden layers.

\noindent\textbf{Procedure.} To train the substitute classifier, the adversary uses the
method of~\cite{papernot2016transferability} to augment the seed data for two
rounds; labels are obtained by querying the target classifier.  Adversarial
examples against the substitutes are created by FGSM,
following~\cite{papernot2016transferability}.  As a sanity check, we verify the
performance of the substitute classifiers on the original training and test
sets. Details are in Table~\ref{table:sub} in the Appendix. Sanity checks on
the performance of the substitute classifiers are presented in
Table~\ref{table:sub} in the Appendix.

\subsubsection{Results}

Figure~\ref{fig:exp_bb} shows the results. For all algorithms, black box attacks are less effective than white box, which corroborates the results
of~\cite{papernot2016transferability}, who observed that black-box attacks are less successful against nearest neighbors. We also find that the kernel substitute attack
is more effective than the neural network substitute, which is
expected as kernel classifiers have similar structure to nearest neighbors. Finally, for Halfmoon and Abalone, our algorithm outperforms the baselines for both attacks; however, for MNIST neural network substitute, our algorithm does not perform as well for small attack radii. This again confirms the theoretical predictions that our algorithm's performance is better when the training set is large relative to the data dimension -- the setting in which nearest neighbors work well in general.

\subsection{Discussion}
The results show that our algorithm performs either better than or about the same as standard baselines against popular white box and black box attacks. As expected from our theoretical results, it performs better for denser datasets which have high or medium amounts of training data relative to the dimension, and either slightly worse or better for sparser datasets, depending on the attack radius. Since non-parametric methods such as nearest neighbors are mostly used for dense data, this suggests that our algorithm has good robustness properties even with reasonable amounts of training data.

\section{Conclusion}
\label{sec:conc}

We introduce a novel theoretical framework for learning robust to adversarial examples, and introduce notions of distributional and finite-sample robustness. We use these notions to analyze a non-parametric classifier, $k$-nearest neighbors, and introduce a novel modified $1$-nearest neighbor algorithm that has good robustness properties in the large sample limit. Experiments show that these properties are still retained for reasonable data sizes. 

Many open questions remain. The first is to close the gap in analysis of $k$-nearest neighbors for $k$ in between our two regimes. The second is to develop nearest neighbor algorithms with better robustness guarantees. Finally, we believe that our work is a first step towards a comprehensive analysis of how the size of training data affects robustness; we believe that an important line of future work is to carry out similar analyses for other supervised learning methods.

\section*{Acknowledgement}
We thank NSF under IIS 1253942 for support. This work was also partially supported by ARO under contract
number W911NF-1-0405. We thank all anonymous reviewers for their constructive comments. 

We also thank Rabanus Derr from University of Tübingen for pointing about a minor mistake in the proof of Lemma~\ref{lem:redtorobustness}. The proof is fixed in the latest version.

\bibliographystyle{plain}
\bibliography{robustnn,somesh,knn}

\clearpage
\appendix
\section{Proofs from Section~\ref{sec:nnanalysis}}

\label{sec:appendixnn}

\subsection{Proofs for Constant $k$}

\begin{proof} (Of Theorem~\ref{thm:knnlowerbound})
To show convergence in probability, we need to show that for all $\epsilon, \delta > 0$, there exists an $n(\epsilon, \delta)$ such that $\Pr( \rho(A_k(S_n, \cdot), x) \geq \epsilon) \leq \delta$ for $n \geq n_0(\epsilon, \delta)$. 

The proof will again proceed in two stages. First, we show in Lemma~\ref{lem:knnnbd} that if the conditions in the statement of Theorem~\ref{thm:knnlowerbound} hold, then there exists some $n(\epsilon, \delta)$ such that for $n \geq n(\epsilon, \delta)$, with probability at least $1 - \delta$, there exists two points $x_+$ and $x_-$ in $B(x, \epsilon)$ such that (a) all $k$ nearest neighbors of $x_+$ have label $1$, (b) all $k$ nearest neighbors of $x_-$ have label $0$, and (c) $x_+ \neq x_-$.

Next we show that if the event stated above happens, then $\rho(A_k(S_n, \cdot), x) \leq \epsilon$. This is because $A_k(S_n, x_+) = 1$ and $A_k(S_n, x_-) = 0$. No matter what $A_k(S_n, x)$ is, we can always find a point $x'$ that lies in $\{x_+, x_-\} \subset B(x, \epsilon)$ such that the prediction at $x'$ is different from $A_k(S_n, x)$. 
\end{proof}

\begin{lemma}
\label{lem:knnnbd}
If the conditions in the statement of Theorem~\ref{thm:knnlowerbound} hold, then there exists some $n(\epsilon, \delta)$ such that for $n \geq n(\epsilon, \delta)$, with probability at least $1 - \delta$, there are two points $x_+$ and $x_-$ in $B(x, \epsilon)$ such that (a) all $k$ nearest neighbors of $x_+$ have label $1$, (b) all $k$ nearest neighbors of $x_-$ have label $0$, and (c) $x_+ \neq x_-$.
\end{lemma}

\begin{proof}(Of Lemma~\ref{lem:knnnbd})
The proof consists of two major components. First, for large enough $n$, with high probability there are many disjoint balls in the neighborhood of $x$ such that each ball contains at least $k$ points in $S_n$. Second, with high probability among these balls, there exists a ball such that the $k$ neareast neighbors of its center all have label $1$. Similarly, there exists a ball such that the $k$ nearest neighbor of its center all have label $0$.

Since $\mu$ is absolutely continuous with respect to Lebesgue measure in the neighborbood of $x$ and $\eta$ is continuous, then for any $m \in \mathbb{Z}+$, we can always find $m$ balls $B(x_1, r_1), \cdots, B(x_m, r_m)$ such that (a) all $m$ balls are disjoint, and (b) for all $i \in \{1, \cdots, m\}$, we have $x_i \in B(x, \epsilon)$, $\mu(B(x_i, r_i)) > 0$ and $\eta(x) \in (0,1)$ for $x \in B(x_i, r_i)$. For simplicity, we use $B_i$ to denote $B(x_i, r_i)$ and $c_i(n)$ to denote the number of points in $B_i \bigcap S_n$. Also, let $\mu_{\min} = \min_{i \in \{1, \cdots, m\}}\mu(B_i)$. Then by Hoeffding's inequality, for each ball $B_i$ and for any $n > \frac{k+1}{\mu_{\min}}$,
$$
\Pr[c_i(n) < k] \leq \exp(-2n\mu_{\min}^2 / (k+1)^2),
$$
where the randomness comes from drawing sample $S_n$.
Then taking the union bound over all $m$ balls, we have
\begin{equation}
\Pr[\exists i \in \{1, \cdots, m\}\ \mbox{such that}\ c_i(n) < k] \leq m\exp(-2n\mu_{\min}^2 / (k+1)^2),
\end{equation}
which implies that when $n > \max \left(\frac{k+1}{\mu_{\min}}, \frac{[\log m - \log(\delta/3)](k+1)^2}{\mu_{\min}^2} \right)$, with probability at least $1-\delta/3$, each of $B_1, \cdots, B_m$ contains at least $k$ points in $S_n$.

An important consequence of the above result is that with probability at least $1-\delta/3$, the set of $k$ nearest neighbors of each center $x_i$ of $B_i$ is completely different from another center $x_j$'s, so the labels of $x_i$'s $k$ nearest neighbors are independent of the labels of $x_j$'s $k$ nearest neighbors.

Now let $\eta_{\min, +} = \min_{x \in B_1\bigcup\cdots\bigcup B_m} \eta(x)$ and $\eta_{\min, -} = \min_{x \in B_1\bigcup\cdots\bigcup B_m} (1-\eta(x))$. Both $\eta_{\min, +}$ and $\eta_{\min, -}$ are greater than 0 by the construction requirements of $B_1, \cdots, B_m$. For any $x_i$,
$$
\Pr[\mbox{$x_i$'s $k$ nearest neighbors all have label $1$}] \geq \eta_{\min, +}^k
$$
Then,
\begin{equation}
\Pr[\mbox{$\exists i\in \{1,\cdots, m\}$} 
\mbox{s.t. $x_i$'s $k$ nearest neighbor all have label $1$}] 
 \geq 
1 - (1-\eta_{\min, +}^k)^m,
\end{equation}
which implies when $m \geq \frac{\log \delta/3}{\log(1 - \eta_{\min, +}^k)}$, with probability at least $1-\delta/3$, there exists an $x_i$ s.t. its $k$ nearest neighbors all have label $1$. This $x_i$ is our $x_+$.

Similarly,
\begin{equation}
\Pr[\mbox{$\exists i\in \{1,\cdots, m\}$}
\mbox{ s.t. $x_i$'s $k$ nearest neighbor all have label $0$}]
\geq 1 - (1-\eta_{\min, -}^k)^m,
\end{equation}
and when $m \geq \frac{\log \delta/3}{\log(1 - \eta_{\min, -}^k)}$, with probability at least $1-\delta/3$, there exists an $x_i$ s.t. its $k$ nearest neighbors all have label $0$. This $x_i$ is our $x_-$.

Combining the results above, we show that for
$$n > \max \left(\frac{k+1}{\mu_{\min}}, \frac{[\log m - \log(\delta/3)](k+1)^2}{\mu_{\min}^2} \right),$$ 
$$m \geq \max \left(\frac{\log \delta/3}{\log(1 - \eta_{\min, +}^k)}, \frac{\log \delta/3}{\log(1 - \eta_{\min, -}^k)} \right),$$ 
with probability at least $1-\delta$, the statement in Lemma~\ref{lem:knnnbd} is satisfied.
\end{proof}

\subsection{Theorem and proof for k-nn robustness lower bound.}
Theorem~\ref{thm:knnlowerbound} shows that k-NN is inherently non-robust in the low $k$ regime if $\eta(x)\in (0,1)$. 
On the contrary, k-NN can be robust at $x$ if $\eta(x)\in \{0,1\}$.
We define the $r$-robust $(p,\Delta)$-interior as follows:
\begin{eqnarray} \label{eqn:defint_eq}
    \hat{\calX}^+_{r, \Delta, p} & = & \{ x \in supp(\mu) | \forall x' \in B^o(x, r), \nonumber \\ && \forall x'' \in B(x', r_p(x')),  \eta(x'') \geq 1/2+\Delta\} \nonumber \\
    \hat{\calX}^-_{r, \Delta, p} & = & \{ x \in supp(\mu) | \forall x' \in B^o(x, r), \nonumber \\ && \forall x'' \in B(x', r_p(x')),  \eta(x'') \leq 1/2-\Delta\} \nonumber
\end{eqnarray}
The definition is similar to the strict $r$-robust $(p,\Delta)$-interior in Section~\ref{sec:defense}, except replacing $<$ and $>$ with $\leq$ and $\geq$.
Theorem~\ref{thm:knnlowerbound01} show that k-NN is robust at radius $r$ in the $r$-robust $(1/2, p)$-interior with high high probability. Corollary~\ref{thm:knnlowerboundrate} shows the finite sample rate of the robustness lowerbound.

\begin{theorem}
\label{thm:knnlowerbound01}
Let $x \in \calX \cap \supp(\mu)$ such that (a) $\mu$ is absolutely
continuous with respect to the Lebesgue measure (b) $\eta(x) \in \{0,
1\}$. Then, for fixed $k$, there exists an $n_0$ such that for $n \geq n_0$, 
$$\Pr[\rho(A_k(S_n, \cdot), x) \geq r] \geq 1-\delta$$
for all $x$ in 
$\hat{\mathcal{X}}^+_{r,1/2,p}\bigcup \hat{\mathcal{X}}^-_{r,1/2,p}$
for all $p>0, \delta>0$.

In addition, with probability at least $1-\delta$,
the astuteness of the k-NN classifier is at least:
$$\mathbb{E}(\mathbf{1}(X\in \hat{\calX}^+_{r,1/2,p}\bigcup \hat{\calX}^-_{r,1/2,p}))$$
\end{theorem}

\begin{proof}
The k-NN classifier $A_k(S_n, \cdot)$ is robust at radius $r$ at $x$
if for every $x'\in B^o(x, r)$, a) there are $k$ training points in 
$B(x',r_p(x'))$, and b) more than $\lfloor k/2 \rfloor$ of them have
the same label as $A_k(S_n, x)$.
Without loss of generality, we look at a point $x \in \hat{\calX}^+_{r,1/2,p}$.
The second condition is satisfied since $\eta(x)=1$ for all training points
in $B(x', r_p(x'))$ 
by the definition of $\hat{\calX}^+_{r,1/2,p}$. 

It remains to check the first condition.
Let $B$ be a ball in $\mathbb{R}^d$ and $n(B)$ be the number of training points in $B$.
Lemma~16 of~\cite{CD10} suggests that with probability at least
$1-\delta$, for all $B$ in $\mathbb{R}^d$,
\begin{equation}
\label{eqn:cd10}
\mu(B)    
\geq \frac{k}{n} + \frac{C_o}{n}
\left(d\log n + \log \frac{1}{\delta} + \sqrt{k\left(d\log n + \log\frac{1}{\delta}\right)}\right)
\end{equation}
implies $n(B) \geq k$,
where $C_o$ is a constant term.
Let $B=B(x', r_p(x'))$. By the definition of $r_p$, $\mu(B)\geq p > 0$.
Then as $n \to \infty$, Inequality~\ref{eqn:cd10} will eventually be satisfied,
which implies $B$ contains at least $k$ training points.
The first condition is then met.

The astuteness result follows because $A_k(S_n, x)=y=1$ in $\hat{\calX}^+_{r,1/2,p}$ 
and $A_k(S_n, x)=y=0$ in $\hat{\calX}^-_{r,1/2,p}$ with probability 1.  
\end{proof}

\begin{corollary}
\label{thm:knnlowerboundrate}
For $n\geq \max(10^4, c_{d,k,\delta}^4/[(k+1)^2p^2])$ where 
$$c_{d,k,\delta}=4(d+1) + \sqrt{16(d+1)^2 + 8(\ln(8/\delta)+k+1)}$$,
with probability at least $1-2\delta$,
$\rho(A_k(S_n, x)) \geq r$ for all $x$ in 
$\hat{\calX}^+_{r,1/2,p}\bigcup \hat{\calX}^-_{r,1/2,p}$\ \ and for all $p>0, \delta>0$. 

In addition, with probability at least $1-2\delta$,
the astuteness of the k-NN classifier is at least:
$$\mathbb{E}(\mathbf{1}(X\in \hat{\calX}^+_{r,1/2,p}\bigcup \hat{\calX}^-_{r,1/2,p}))$$
\end{corollary}
\begin{proof}
Without loss of generality, we look at a point $x \in \hat{\mathcal{X}}^+_{r,1/2,p}$.
Let $B=B(x', r_p(x'))$, $J(B)=\mathbb{E}(Y\cdot \mathbf{1}(X\in B))$
 and $\hat{J}(B)$ be the empirical estimation of $J(B)$.
Notice that $\hat{J}(B)n$ is the number of training points in $B$, because $\eta(x)=1$ for all $x \in B$ by the definition of $r$-robust $(1/2,p)$-interior.
It remains to find a threshold $n$ such that for all $n'>n$,
\begin{equation}
\hat{J}(B) \geq (k+1)/n'
\end{equation} 
By Lemma~\ref{lem:concJn}, with probability $1-2\delta$, 
\begin{equation}
\hat{J}(B) \geq p-2\beta_n\sqrt{p} - 2\beta^2_n
\end{equation} 
for all $B\in \mathbb{R}^d$.
\end{proof}
Therefore it suffices to find a threshold $n$ that satisfies
\begin{equation}
p - 2\beta_n\sqrt{p} - 2\beta^2_n \geq (k+1)/n,
\end{equation}
where $\beta_n = \sqrt{(4/n)((d+1) \ln 2n + \ln(8/\delta))}$.

Solving this quadratic inequality yields
\begin{equation}
\beta_n \leq \frac{-\sqrt{p}+\sqrt{3p+(k+1)/n}}{2},
\end{equation}
which can be re-written as
\begin{equation}
(8/\sqrt{n})[(d+1)\ln(2n) + \ln(8/\delta) + (k+1)/8] \leq \sqrt{(k+1)p}
\end{equation}
by substituting the expression for $\beta_n$. 
This inequality does not admit an analytic solution.
Nevertheless, we observe that $n^{1/4} \geq \ln(2n)$ for all $n\geq 10^4$.
Therefore it suffices to find an $n\geq 10^4$ such that
\begin{equation}
\label{eqn:complex}
(8/\sqrt{n})[(d+1)n^{1/4} + \ln(8/\delta) + (k+1)/8] \leq \sqrt{(k+1)p}.
\end{equation}
Let $m = n^{1/4}$. Inequality~\ref{eqn:complex} can be re-written as
\begin{equation}
\sqrt{(k+1)p}m^2 - 8(d+1)m - (8\ln(8/\delta)+(k+1)) \geq 0.
\end{equation}

Solving this quadratic inequality with respct to $m$ gives
\begin{equation}
m \geq \frac{4(d+1) + \sqrt{16(d+1)^2 + 8(\ln(8/\delta)+k+1)}}{\sqrt{(k+1)p}}.
\end{equation}

Letting 
$$c_{d,k,\delta} = 4(d+1) + \sqrt{16(d+1)^2 + 8(\ln(8/\delta)+k+1)}$$, 
we find a desired threshold 
\begin{equation}
n = \max(10^4, m^4) \geq \max(10^4, c_{d,k,\delta}^4/[(k+1)^2p^2]).
\end{equation}

The astuteness result follows in a similar way to Theorem~\ref{thm:knnlowerbound01}.

\subsection{Proofs for High $k$}

\subsubsection{Robustness of the Bayes Optimal Classifier}
\begin{proof} (Of Theorem~\ref{thm:bayesoptrobustness})
Suppose $x \in \calX^+_{r, 0, 0}$. Then, $g(x) = 1$. Consider any $x' \in B^o(x, r)$; by definition, $\eta(x') > 1/2$, which implies that $g(x') =1$ as well. Thus, $\rho(g, x) \geq r$. The other case ($x \in \calX^-_{r, 0, 0})$ is symmetric. 

Consider an $x \in \calX^+_{r, 0, 0}$ (the other case is symmetric). We just showed that $g$ has robustness radius $\geq r$ at $x$. Moreover, $p(y=1=g(x)|x) = \eta(x)$; therefore, $g$ predicts the correct label at $x$ with probability $\eta(x)$. The theorem follows by integrating over all $x$ in $\calX^+_{r, 0, 0} \cup \calX^-_{r, 0, 0}$. 
\end{proof}

\subsubsection{Robustness of $k$-Nearest Neighbor}

We begin by stating and proving a more technical version of Theorem~\ref{thm:knnrobustness}. 

\begin{theorem}
For any $n$ and data dimension $d$, define:
\begin{eqnarray*}
a_n & = & \frac{C_0}{n} (d \log n + \log(1/\delta)) \\
b_n & = & C_0 \sqrt{\frac{d \log n + \log(1/\delta)}{n}} \\
\beta_n & = & \sqrt{(4/n)((d+1) \ln 2n + \ln(8/\delta))}
\end{eqnarray*}
where $C_0$ is the constant in Theorem 15 of~\cite{CD10}. Now, pick $k_n$ and $\Delta_n$ so that $\Delta_n \rightarrow 0$ and the following condition is satisfied:
\[ \frac{k_n}{n} \geq \frac{2 \beta_n + b_n + \sqrt{(2 \beta_n + b_n)^2 + 2 \Delta_n (2 \beta_n^2 + a_n)}}{ \Delta_n} \]
and set
\begin{eqnarray*}
p_n & = & \frac{k_n}{n} + \frac{C_0}{n}\Big{(}d \log n + \log(1/\delta) \\
&& + \sqrt{k_n (d \log n + \log(1/\delta)}\Big{)} 
\end{eqnarray*}

Then, with probability $\geq 1 - 3 \delta$, $k_n$-NN has robustness radius $r$ at all $x \in \calX^+_{r, \Delta_n, p_n} \cup \calX^-_{r, \Delta_n, p_n}$. In addition, with probability $\geq 1 - \delta$, the \raccuracy\ of $k_n$-NN is at least:
\[ \E[ \eta(X) \cdot \mone(X \in \calX^+_{r, \Delta_n, p_n})] + \E (1 - \eta(X)) \cdot \mone(X \in \calX^-_{r, \Delta_n, p_n})]\]
\label{thm:knnrobustastute}
\end{theorem}

Before we prove Theorem~\ref{thm:knnrobustastute}, we need some definitions and lemmas. 

For any Euclidean ball $B$ in $\mathbb{R}^d$, define $J(B) = \mathbb{E}[Y \cdot \mone (X \in B)]$ and $\hat{J}(B)$ as the corresponding empirical quantity.

\begin{lemma}\label{lem:concJn}
With probability $\geq 1 - 2\delta$, for all balls $B$ in $\mathbb{R}^d$, we have:
\[ |J(B) - \hat{J}(B)| \leq 2\beta_n^2 + 2\beta_n \min(\sqrt{J(B)}, \sqrt{\hat{J}(B)}),\]
where $\beta_n = \sqrt{(4/n)((d+1) \ln 2n + \ln(8/\delta))}$.
\end{lemma}

\begin{proof}(Of Lemma~\ref{lem:concJn})
Consider the two functions: $h_B^+(x, y) = \mone(y = 1, x \in B)$ and $h_B^-(x, y) = \mone(y = -1, x \in B)$. From Lemma~\ref{lem:vcj}, both $h_B^+$ and $h_B^-$ are $0/1$ functions with VC dimension at most $d+1$. Additionally, $J(B) = \E[h_B^+] - \E[h_B^-]$. Applying Theorem 15 of~\cite{CD10}, along with an union bound gives the lemma.
\end{proof}

\begin{lemma}\label{lem:vcj}
For an Euclidean ball $B$ in $\bbR^d$, define the function $h_B^+: \bbR^d \times \{-1, +1\} \rightarrow \{ 0, 1 \}$ as:
\[ h^+_B(x, y) = 1(y = 1, x \in B) \]
and let $\calH_B = \{ h^+_B \}$ be the class of all such functions. Then the VC-dimension of $\calH_B$ is at most $d+1$.
\end{lemma}

\begin{proof}(Of Lemma~\ref{lem:vcj})
Let $U$ be a set of $d+2$ points in $\bbR^d$; as the VC dimension of balls in $\bbR^d$ is $d+1$, $U$ cannot be shattered by balls in $\bbR^d$. Let $U_L = \{ (x, y) | x \in U \}$ be a labeling of $U$ that cannot be achieved by any ball (with pluses inside and minuses outside); the corresponding $d+1$-dimensional points cannot be labeled accordingly by $h_B^+$. Since $U$ is an arbitrary set of $d+2$ points, this implies that any set of $d+2$ points in $\bbR^d \times \{-1, +1\}$ cannot be shattered by $\calH_B$. The lemma follows. 
\end{proof}

\begin{lemma} \label{lem:densityconc}
Let $\delta_p = \frac{C_0}{n}\left(d \log n + \log(1/\delta) + \sqrt{k (d \log n + \log(1/\delta)}\right)$. Then, with probability $\geq 1 - \delta$, for all $x$, $\| x - X_{(k+1)}(x)\| \leq r_{k/n + \delta_p}(x)$, and $\mu(B(x, \|x - X_{(k+1)}(x)\|)) \geq \frac{k}{n} - \delta_p$.
\end{lemma}

\begin{proof}(Of Lemma~\ref{lem:densityconc})
Observe that by definition for any $x$, $r_p$ is the smallest $r$ such that $\mu(B(x, r_p(x)) \geq p$. The rest of the proof follows from Lemma 16 of~\cite{CD10}. 
\end{proof}

\begin{proof}(Of Theorem~\ref{thm:knnrobustastute})

From Lemma~\ref{lem:densityconc}, by uniform convergence of $\hat{\mu}$, with probability $\geq 1 - \delta$, for all $x'$, $\| x' - X^{(k_n)}(x') \| \leq r_{p_n}(x')$ and $\mu(B(x, \|x - X^{(k_n)}(x)\|)) \geq \frac{k_n}{n} - \delta_p$. If $x' \in \calX^{+}_{r, \Delta_n, p_n}$, this implies that for all $\tilde{x} \in B(x', X^{(k_n)}(x'))$, $\eta(\tilde{x}) \geq 1/2 + \Delta$. Therefore, for such an $x'$, $J(B(x', X^{(k_n)}(x'))) \geq (\frac{1}{2} + \Delta_n)\mu(B(x', X^{(k_n)}(x'))) \geq (\frac{1}{2} + \Delta_n) (k_n/n - \delta_p)$. Since for $B(x', X^{(k_n)}(x'))$, $\hat{\mu}(B(x', X^{(k_n)}(x'))) = \frac{k_n}{n}$, $\min(\hat{J}, J) \leq \frac{k}{n}$. Thus we can apply Lemma~\ref{lem:concJn} to conclude that
\[ \hat{J}(B) > J(B) - 2 \beta_n^2 - 2 \beta_n \sqrt{k_n/n} > \frac{k_n}{2n},\]
which implies that $\hat{Y}(B) = \frac{1}{k_n} \sum_{i=1}^{k_n} Y^{(i)}(x) = \frac{n}{k_n}\hat{J}(B) > \frac{1}{2}$. The first part of the theorem follows.

For the second part, observe that for an $x \in \calX^{+}_{r, \Delta_n, p_n}$, the label $Y$ is equal to $+1$ with probability $\eta(x)$ and for an $x \in \calX^{-}_{r, \Delta_n, p_n}$, the label $Y$ is equal to $-1$ with probability $1 - \eta(x)$. Combining this with the first part completes the proof.
\end{proof}

\begin{figure*}[!t]
\includegraphics[scale = 0.28]{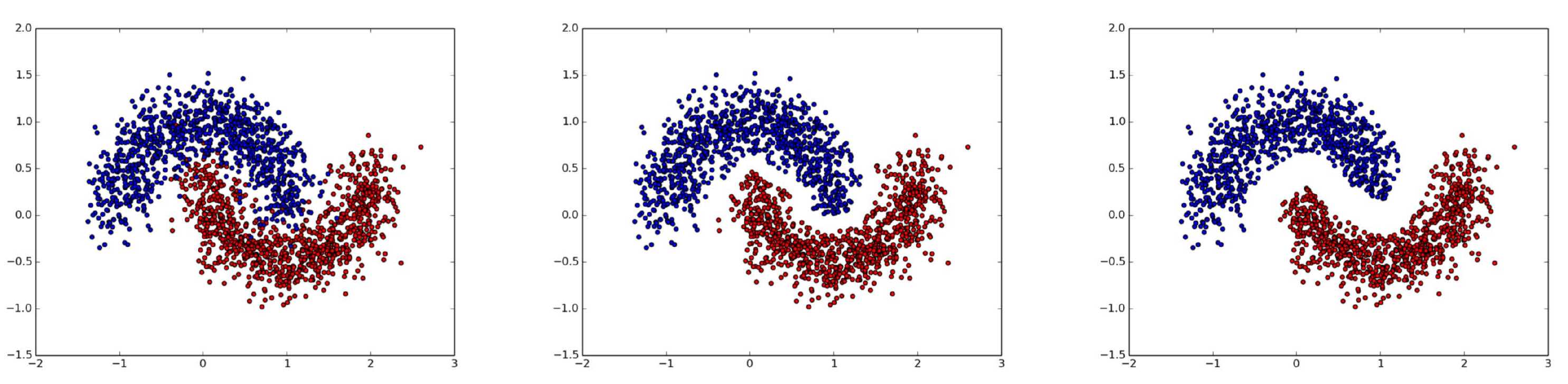}
\caption{Visualization of the halfmoon dataset. 1) Training sample of size $n = 2000$, 2) subset selected by Robust\_1NN with defense radius $r = 0.1$, 3) subset selected by Robust\_1NN with defense radius $r = 0.2$.}
\label{fig:halfmoon}
\end{figure*}

\begin{figure*}[!t]
\centering
\includegraphics[scale=0.28]{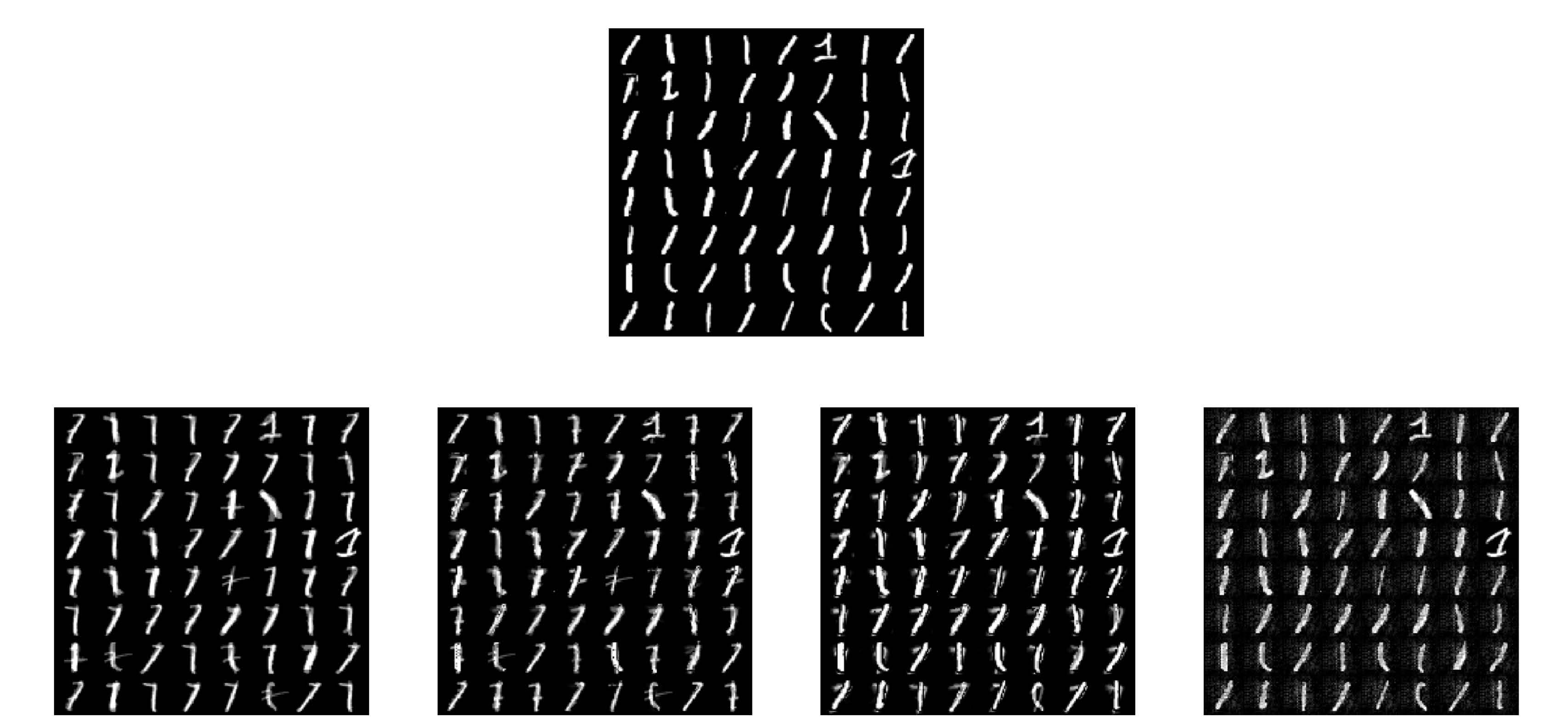}
\caption{Adversarial examples of MNIST digit 1 images created by different attack methods. \textit{Top row:} clean digit 1 test images. \textit{Middle row from left to right:} 1) direct attack, 2) white-box kernel attack. \textit{Bottom row from left to right:} 1) black-box kernel attack, 2) black-box neural net substitute attack.}
\label{fig:adv}
\end{figure*}

\begin{center}
\begin{table*}[!t]
\centering
\setlength\tabcolsep{1.5pt}
\small
\caption{An evaluation of the black-box substitute classifier. Each black-box substitute is evaluated by: 1) its accuracy on the its training set, 2) its accuracy on the test set, and 3) the percentage of predictions agreeing with the target classifier on the test set. A combination of high test accuracy and consistency with the original classifier indicates the black-box model emulates the target classifier well.}
\vspace{10pt}
\label{table:sub}
\scalebox{1}{
\begin{tabular}{ c | c | c | c | c}
  \multicolumn{5}{c}{Abalone}\\
  \hline
  & target $f$& \% training & \% test& \% test$f$\\
  & & accuracy & accuracy & same as $f$\\
 \hline
 \multirow{4}{4em}{Kernel} &  StandardNN & 100\% & 61.3\% &72.6\%\\
  & RobustNN & 100\% &62.5\% & 90.9\%\\  
 & ATNN & 100\% & 61.4\%  &  73.7\%\\
 & ATNN-All & 100\% & 63.5\%  &  73.5\%\\
 \hline
 \multirow{4}{4em}{Neural Nets} & StandardNN & 69.1\% & 68.9\% & 68.6\% \\
 & RobustNN & 87.2\% & 64.1\% &  86.9\%\\
 & ATNN & 68.8\% & 68.4\% &
 68.4\%\\
 & ATNN-All & 66.5\% & 65.0\%  &  66.6\%\\
 \hline
\end{tabular}
}
\scalebox{1}{
\begin{tabular}{ c | c | c | c | c}
  \multicolumn{5}{c}{Halfmoon}\\
  \hline
  & target $f$& \% training & \% test& \% test\\
  & & accuracy & accuracy & same as $f$\\
 \hline
 \multirow{4}{4em}{Kernel} &  StandardNN & 95.9\% & 95.6\% &95.5\%\\
  & RobustNN & 97.7\% & 94.9\% & 97.6\%\\  
 & ATNN & 96.4\% & 95.1\%  &  96.0\%\\
& ATNN-All & 97.6\% & 96.8\%  &  97.3\%\\
 \hline
 \multirow{4}{4em}{Neural Nets} & StandardNN & 94.5\% & 94.0\% & 94.4\% \\
 & RobustNN & 94.2\% & 90.5\% &  94.1\%\\
 & ATNN & 95.3\% & 94.2\% &
 95.2\%\\
& ATNN-All & 96.9\% & 96.2\%  &  96.5\%\\
 \hline
\end{tabular}
}

\vspace{10pt}
\scalebox{1}{
\begin{tabular}{ c | c | c | c | c}
  \multicolumn{5}{c}{MNIST 1v7}\\
  \hline
  & target $f$& \% training& \% test& \% test\\
  & & accuracy & accuracy & same as $f$\\
 \hline
 \multirow{4}{4em}{Kernel} &  StandardNN & 100\% & 98.9\% &99.3\%\\
  & RobustNN & 100\% & 95.4\% & 97.6\%\\  
 & ATNN & 100\% & 98.9\%  &  99.3\%\\
& ATNN-All & 100\% & 98.7\%  &  99.3\%\\
 \hline
 \multirow{4}{4em}{Neural Nets} & StandardNN & 99.9\% & 98.9\% & 99.1\% \\
 & RobustNN & 99.8\% & 94.8\% &  98.7\%\\
 & ATNN & 100\% & 98.8\% &
 99.2\%\\
& ATNN-All & 99.7\% & 98.9\%  &  99.3\%\\
 \hline
\end{tabular}
}
\end{table*}
\end{center}

\section{Proofs from Section~\ref{sec:defense}}

We begin with a statement of Chernoff Bounds that we use in our calculations.

\begin{theorem}~\cite{UpfalMitzenmacher00}
Let $X_i$ be a $0/1$ random variable and let $X = \frac{1}{m}\sum_{i=1}^{m} X_i$. Then,
\[ \Pr( | X - \E[X]| \geq \delta) \leq e^{-m \delta^2/2} + e^{-m \delta^2/3} \leq 2 e^{-m \delta^2/3}\]
\label{thm:chernoff}
\end{theorem}

\begin{lemma} \label{lem:redseparated}
Suppose we run Algorithm~\ref{alg:robust_1nn} with parameter $r$. Then, the points marked as red by the algorithm form an $r$-separated subset of the training set.
\end{lemma}
\begin{proof}
Let $f(x_i)$ denote the output of Algorithm~\ref{alg:confident_label} on $x_i$. If $(x_i, 1)$ is a Red point, then $f(x_i) = 1 = f(x_j)$ for all $x_j \in B(x, r)$; therefore, $(x_j, -1)$ cannot be marked as Red by the algorithm as $f(x_j) \neq y_j$. The other case, where $(x_i, -1)$ is a Red point is similar.
\end{proof}

\begin{lemma}\label{lem:redtorobustness}
Let $x \in \calX$ such that Algorithm~\ref{alg:robust_1nn} finds a Red $x_i$ within $B^o(x, \tau)$. Then, Algorithm~\ref{alg:robust_1nn} has robustness radius at least $r - 2 \tau$ at $x$.
\end{lemma}

\begin{proof}
For all $x' \in B(x, \tau)$, we have:
\[ \| x' - x_i\| \leq \| x - x_i\| + \| x - x' \| < 2 \tau \]
Since $x_i$ is a Red point, from Lemma~\ref{lem:redseparated}, any $x_j$ in training set output by Algorithm~\ref{alg:robust_1nn} with $y_j \neq y_i$ must have the property that $\|x_i - x_j \| > 2r$. Therefore, 
\[ \| x' - x_j\| \geq \| x_i - x_j \| - \|x' - x_i\| > 2r - 2 \tau \]
Therefore, Algorithm~\ref{alg:robust_1nn} will assign $x'$ the label $y_i$. The lemma follows.
\end{proof}

\begin{lemma}\label{lem:existsred}
Let $B$ be a ball such that: (a) for all $x \in B$, $\eta(x) > \frac{1}{2} + \Delta$ and (b) $\mu(B) \geq \frac{2C_0}{n} (d \log n + \log(1/\delta))$. Then, with probability $\geq 1 - \delta$, all such balls have at least one $x_i$ such that $x_i \in |B \cap X_n|$ and $y_i = 1$.
\end{lemma}

\begin{proof}
Observe that $J(B) \geq \frac{C_0}{n} (d \log n + \log(1/\delta))$. Applying Theorem 16 of~\cite{CD10}, this implies that $\hat{J}(B) > 0$, which gives the theorem.
\end{proof}

\begin{lemma}
Fix $\Delta$ and $\delta$, and let $k_n = \frac{3 \log (2n /\delta)}{\Delta^2}$. Additionally, let 
\[ p_n = \frac{k_n}{n} + \frac{C_0}{n}(d \log n + \log(1/\delta) + \sqrt{k_n(d \log n + \log(1/\delta)}),\] 
where $C_0$ is the constant in Theorem 15 of~\cite{CD10}. Define:
\begin{eqnarray*}
 S_{RED} & = & \{ (x_i, y_i) \in S_n | x_i \in \calX^+_{r, \Delta, p_n} \cup \calX^-_{r, \Delta, p},\\
&&  y_i = \frac{1}{2}sgn\left(\eta(x_i) - \frac{1}{2}\right) + \frac{1}{2} \} \end{eqnarray*}
Then, with probability $\geq 1 - \delta$, all $(x_i, y_i) \in S_{RED}$ are marked as Red by Algorithm~\ref{alg:robust_1nn} run with parameters $r$, $\Delta$ and $\delta$.
\label{lem:markred}
\end{lemma}

\begin{proof}
Consider a $(x_i, y_i) \in S_{RED}$ such that $x_i \in X_n \cap \calX^+_{r, \Delta, p_n}$, and consider any $(x_j, y_j) \in S_n$ such that $x_j \in B(x_i, r)$. From Lemma~\ref{lem:densityconc}, for all such $x_j$, $\| x_j - X^{(k_n)}(x_j)\| \leq r_{p_n}(x_j)$; this means that all $k_n$-nearest neighbors $x''$ of such an $x_j$ have $\eta(x'') > \frac{1}{2} + \Delta$. 

Therefore, $\E[ \sum_{l = 1}^{k_n} Y^{(l)}(x_j) ] \geq k_n(1/2 + \Delta)$; by Theorem~\ref{thm:chernoff}, this means that for a specific $x_j$, $\Pr(\sum_{l=1}^{k_n} Y^{(l)}(x_j) < 1/2) \leq 2e^{-k_n \Delta^2/3}$, which is $\leq \delta/n$ from our choice of $k_n$. By an union bound over all such $x_j$, with probability $\geq 1 - \delta$, we see that Algorithm~\ref{alg:confident_label} reports the label $g(x_i)$ on all such $x_i$, which is the same as $y_i$ by the definition of interiors; $x_i$ therefore gets marked as Red.  
\end{proof}

Finally, we are ready to prove the main theorem of this section, which is a slightly more technical form of Theorem~\ref{thm:defense}. 

\begin{theorem}
Fix a $\Delta_n$, and pick $k_n$ and $p_n$ as in Lemma~\ref{lem:markred}. Suppose we run Algorithm~\ref{alg:robust_1nn} with parameters $r$, $\Delta_n$ and $\delta$. Consider the set:
\begin{eqnarray*}
 X_R & = & \Bigg{\{} x \Big{|} x \in \calX^+_{r + \tau, \Delta_n, p_n} \cup \calX^-_{r + \tau, \Delta_n, p_n}, \\
&& \mu(B(x, \tau)) \geq \frac{2C_0}{n}(d \log n + \log(1/\delta)) \Bigg{\}} ,\end{eqnarray*}
where $C_0$ is the constant in Theorem 15 of~\cite{CD10}. Then, with probability $\geq 1 - 2\delta$ over the training set, Algorithm~\ref{alg:robust_1nn} has robustness radius $\geq r - 2 \tau$ on $X_R$. Additionally, its \raccuracy\ at radius $r - 2 \tau$ is at least $\E[ \eta(X) \cdot \mone(X \in \calX^+_{r + \tau, \Delta_n, p_n})] + \E[ (1 - \eta(X)) \cdot \mone(X \in \calX^-_{r + \tau, \Delta_n, p_n})]$. 
\end{theorem}

\begin{proof}
Due to the condition on $\mu(B(x, \tau))$, from Lemma~\ref{lem:existsred}, with probability $\geq 1 - \delta$, all $x \in X_R$ have the property that there exists a $(x_i, y_i)$ in $S_n$ such that $y_i= g(x_i)$ and $x_i \in B(x, \tau)$. Without loss of generality, suppose that $x \in \calX^+_{r + \tau, \Delta_n, p_n}$, so that $\eta(x) > 1/2 + \Delta_n$. Then, from the properties of $r$-robust interiors, this $x_i \in \calX^+_{r, \Delta_n, p_n}$.

From Lemma~\ref{lem:markred}, with probability $\geq 1 - \delta$, this $(x_i, y_i)$ is marked Red by Algorithm~\ref{alg:robust_1nn} run with parameters $r$, $\Delta_n$ and $\delta$. The theorem now follows from an union bound and Lemma~\ref{lem:redtorobustness}.
\end{proof}

\section{Experiment Visualization and Validation}

First, we show adversarial examples created by different attacks on the MNIST dataset in order to illustrate characteristics of each attack.  
Next, we show the subset of training points selected by Algorithm~\ref{alg:robust_1nn} on the halfmoon dataset. The visualization illustrates the intuition behind Algorithm~\ref{alg:robust_1nn} and also validates its implementation. Finally, we validate how effective the black-box subsitute classifiers emulate the target classifier.

\subsection{Adversarial Examples Created by Different Attacks} 
Figure~\ref{fig:adv} shows adversarial examples created on MNIST digit 1 images with attack radius $r = 3$.  
First, we observe that the perturbations added by direct attack, white-box kernel attack and black-box kernel attack are clearly targeted: either a faint horizontal stroke or a shadow of digit 7 are added to the original image. The perturbation budget is used on "key" pixels that distinguish digit 1 and digit 7, therefore the attack is effective.
On the contrary, black-box attacks with neural nets substitute adds perturbation to a large number of pixels.  
While such perturbation often fools a neural net classifier, it is not effective against nearest neighbors. Consider a pixel that is dark in most digit 1 and digit 7 training images; adding brightness to this pixel increases the distance between the test image to training images from both classes, therefore may not change the nearest neighbor to the test image. 

Figure~\ref{fig:adv} also illustrates the break-down attack radius of visual similarity. At $r = 3$, the true class of adversarial examples created by effective attacks becomes ambiguous even to humans.
Our defense is successful as the Robust\_1NN classifiers still have non-trivial classification accuracy at such attack radius.    
Meanwhile, we should not expect robustness against even larger attack radius since the adversarial examples at $r=3$ are already close to the boundary of human perception.

\subsection{Training Subset Selected by Robust\_1NN} 
Figure~\ref{fig:halfmoon} shows the training set selected by Robust\_1NN on a halfmoon training set of size $2000$. On the original training set, we see a noisy region between the two halfmoons where both red and blue points appear. Robust\_1NN cleans training points in this region so as to create a gap between the red and blue halfmoons, and the gap width increases with defense radius $r$.

\subsection{Performance of Black-box Attack Substitutes}
We validate the black-box substitute training process by checking the substitute's accuracy on its training set, the clean test set and the percentage of predictions agreeing with the target classifier on the clean test set. The results are shown in Table~\ref{table:sub}. 
For the halfmoon and MNIST dataset, the substitute classifiers both achieve high accuracy on both the training and test sets, and are also consistent with the target classifier on the test set.   
The subsitutute classifiers do not emulate the target classifier on the Abalone dataset as close as on the other two datasets due to the high noise level in the Abalone dataset. Nonetheless, the substitute classifier still achieve test time accuracy comparable to the target classifier.

\end{document}